\newcommand*{\AISTAT}{}

\newcommand*{\CAMREADY}{}

\ifdefined\ICLR
        \documentclass{article}
        \usepackage{iclr2017_conference,times}
        \ifdefined\CAMREADY
                \iclrfinalcopy
        \fi
\fi

\ifdefined\AISTAT
        \documentclass[twoside]{article}
        \usepackage[round]{natbib}
        \ifdefined\CAMREADY
                \usepackage[accepted]{aistats2018}
        \else
                \usepackage{aistats2018}
        \fi
\fi

\usepackage{url}
\usepackage{color}
\usepackage{tabu,booktabs,multirow}
\usepackage{amsfonts}
\usepackage{amsmath}
\usepackage{mathtools}
\usepackage{relsize}
\usepackage{hyperref}
\usepackage{algorithm}
\usepackage[noend]{algpseudocode}
\usepackage{graphicx}
\usepackage{nicefrac}
\usepackage{bbm}
\usepackage{amsthm}
\usepackage{wrapfig}
\usepackage{diagbox}
\usepackage{subcaption}
\usepackage{caption}
\usepackage[titletoc,title]{appendix}
\usepackage{todonotes}
\usepackage{enumitem}

\newtheorem{lemma}{Lemma}

\newtheorem{theorem}{Theorem}
\newtheorem{proposition}{Proposition}

\newtheorem{claim}{Claim}

\theoremstyle{definition}

\newtheorem{definition}{Definition}

\newcommand{\x}{{\mathbf x}}
\newcommand{\y}{{\mathbf y}}
\newcommand{\z}{{\mathbf z}}

\newcommand{\w}{{\mathbf w}}
\newcommand{\s}{{\mathbf s}}

\newcommand{\aaa}{{\mathbf a}}
\newcommand{\bb}{{\mathbf b}}

\newcommand{\E}{{\mathcal E}}

\newcommand{\EEE}{{\mathbf E}}
\newcommand{\R}{{\mathbb R}}
\newcommand{\N}{{\mathbb N}}

\newcommand{\indc}[1]{\mathbb{I}\left[#1\right]}

\newcommand{\shorteq}{\scalebox{0.8}[1]{\text{$=$}}}
\newcommand{\children}[1]{{\operatorname{ch}\!\left( #1 \right)}}
\newcommand{\scope}[1]{{\operatorname{sc}\!\left( #1 \right)}}
\newcommand{\qnu}[1]{{\operatorname{nu}\!\left( #1 \right)}}
\newcommand{\qde}[1]{{\operatorname{de}\!\left( #1 \right)}}
\newcommand{\qeff}[1]{{\operatorname{eff}\!\left( #1 \right)}}
\newcommand{\qcond}[1]{{\operatorname{cond}\!\left( #1 \right)}}
\newcommand{\CMO}{\operatorname{CMO}}

\renewcommand{\bibsection}{}
\ifdefined\CAMREADY
	
\else
	
\fi

\presetkeys{todonotes}{inline}{}

\begin{document}


\newcommand\gapaftersection{\vspace{-2mm}}
\newcommand\gapbeforesection{\vspace{-2.5mm}}
\newcommand\gapaftersubsection{\vspace{-0.5mm}}
\newcommand\gapbeforesubsection{\vspace{-1mm}}
\setlength{\abovedisplayskip}{4pt}
\setlength{\belowdisplayskip}{4pt}
\setlength{\textfloatsep}{12.0pt plus 0.0pt minus 0.0pt} 


\ifdefined\ICLR
	\title{Sum-Product-Quotient Networks}
	\author{Or Sharir \& Amnon Shashua\\
	The Hebrew University of Jerusalem \\
	\texttt{\{or.sharir,shashua\}@cs.huji.ac.il}
	}
	\maketitle
    \ifdefined\CAMREADY
          \ifdefined\ARXIVONLY
                \lhead{}
            \renewcommand{\headrulewidth}{0pt}
          \fi
    \fi
\fi

\ifdefined\AISTAT
    \twocolumn[
        \aistatstitle{Sum-Product-Quotient Networks}
        \ifdefined\CAMREADY
            \aistatsauthor{ Or Sharir \And Amnon Shashua}
            \aistatsaddress{
                The Hebrew University of Jerusalem \\
                \texttt{or.sharir@cs.huji.ac.il} \And
                The Hebrew University of Jerusalem \\
                \texttt{shashua@cs.huji.ac.il}
            }
        \else
            \aistatsauthor{ Anonymous Author 1 \And Anonymous Author 2}
            \aistatsaddress{ Unknown Institution 1 \And Unknown Institution 2}
        \fi
    ]
    \ifdefined\CAMREADY
        \ifdefined\ARXIVONLY
            \makeatletter
            \renewcommand{\AISTATS@appearing}{}
            \renewcommand{\@copyrightspace}{}
            \makeatother
        \fi
    \fi
\fi


\begin{abstract}
\vspace{-4mm}
We present a novel tractable generative model that extends Sum-Product
Networks (SPNs) and significantly boosts their power. We call it
\emph{Sum-Product-Quotient Networks} (SPQNs), whose  core concept is to
incorporate conditional distributions into the model by direct computation
using quotient nodes, e.g.~$P(A|B){=}\frac{P(A,B)}{P(B)}$. We provide sufficient
conditions for the tractability of SPQNs that generalize and relax the
decomposable and complete tractability conditions of SPNs. These relaxed
conditions give rise to an exponential boost to the expressive efficiency of our
model, i.e.~we prove that there are distributions which SPQNs can compute
efficiently but require SPNs to be of exponential size. Thus, we narrow the gap
in expressivity between tractable graphical models and other Neural
Network-based generative models.
\vspace{-3mm}
\end{abstract}

\gapbeforesection
\section{Introduction} \label{sec:intro}
\gapaftersection

Sum-Product Networks~(SPNs)\citep{Poon:2012vd} are a class of generative models
capable of exact and tractable inference, where the probability function is
directly modelled as a simple computational graph composed of just
\emph{weighted sum} and \emph{product} nodes, known also as Arithmetic
Circuits~\citep{Shpilka:2010cu}, following a strict set of constraints on
its connectivity. SPNs have been applied to solve a wide range of tasks, e.g.
image classification~\citep{Gens:2012vq}, activity
recognition~\citep{Amer:2012uq,Amer:2016hs}, and missing data~\citep{tmm}.
While SPNs have certain advantages in some areas, in cases where expressiveness
is a limiting factor, they fall far behind contemporary generative models such
as those leveraging neural networks as their inference
engine~\citep{JMLR:v17:16-272,vandenOord:2016um,Dinh:2016vt}.

When SPNs were first introduced, it was hypothesized that perhaps all tractable
distributions could be represented efficiently by SPNs. However, this hypothesis
was later proven to be false by~\citet{Martens:2014tr}. More specifically, they
have shown that the uniform distribution on the spanning trees of a complete
graph on $n$ vertices, which is known to be tractable by other methods, cannot
be realized by SPNs, unless their size is exponential in $n$. The reason behind
this limitation is not due to the simple operations on which they are built on,
as any efficiently computable function could be approximated arbitrarily well by
a polynomially-sized arithmetic circuits~\citep{Hoover:1990jj}, but rather its
the strict structural constraints of SPNs that are required for tractability.

In this paper, we introduce an extension to SPNs, which we call
\emph{Sum-Product-Quotient Networks}~(SPQNs for short), that addresses the
limited expressivity of SPNs. The underlying concept behind our extension is to
incorporate conditional probabilities into the model through direct computation,
i.e.~by repeatedly applying the formula $P(A|B) = \frac{P(A,B)}{P(B)}$.
Specifically, we show that by adding a quotient node, i.e.~a node with two
inputs that computes their division, we can relax the structural constraints of
SPNs and still have a model capable of tractable inference, where each internal
node represents a conditional probability over its input variables. Moreover,
we prove that while SPQNs can represent any distribution SPNs can, by virtue
of being their extension, there exists distributions that are efficient for
SPQNs but require SPNs to be of exponential size, proving SPQNs are
exponentially more expressively efficient.

The rest of the article is organized as follows. In sec.~\ref{sec:pre} we
briefly describe the SPNs model and its basic concepts. This is followed by 
sec.~\ref{sec:model} in which we present our SPQNs extension, and prove that the
resulting model is indeed tractable. In sec.~\ref{sec:analysis} we analyze the
expressive efficiency of SPQNs with respect to SPNs. Finally, we discuss the
implications of our model on prior work and our plans for future research in
sec.~\ref{sec:discussion}.

\gapbeforesection
\section{Preliminaries} \label{sec:pre}
\gapaftersection

In this section we give a brief description of Sum-Product Networks~(SPNs).
For simplicity, we limit our description to probability models over binary
variables, where the extension to higher-dimensional or continuous variables
is quite straightforward.

An SPN over binary random variables $X_1, \ldots, X_N$ is a rooted
computational directed acyclic graph, which computes the unnormalized
probability function of the evidence $x_1,\ldots,x_N \in \{0,1,*\}$, denoted by
$\Psi(x_1,\ldots,x_N)$, where $*$ denotes missing variables under which
the SPN computes just the unnormalized marginal of the visible variables. The
leaves of an SPN are univariate indicators of the binary variables,
i.e.~$\indc{x_i=0}$ and $\indc{x_i=1}$, with the special property that for $x_i = *$
all respective indicators of $x_i$ equal $1$. The internal nodes of the SPN
compute either a positive weighted sum or a product, i.e.~an SPN is
an Arithmetic Circuit over the indicator variables defined above. We denote by
$S$ the set of sum nodes, by $P$ the set of product nodes, by $I$ the set of
indicator nodes, and by $V = S \cup P \cup I$ the set of all nodes in the SPN.
For all $v \in V$, we denote by $\children{v}$ the set of children nodes
pointing  to $v$, and define the scope of $v$, denoted by $\scope{v}$, as the
index set of all variables, such that there exists a path starting at an
indicator of a variable, which ends at the node $v$. Formally, we define
$\scope{v} \equiv \{i\}$ for leaf nodes of the $i$-th variables, and otherwise
$\scope{v} \equiv \cup_{c \in \children{v}} \scope{c}$. We denote the function
induced by the sub-graph rooted at $v$ over the variables in $\scope{v}$ by
$\Psi_v(\cdot)$. Last, we define the following structural properties for SPNs:

\begin{definition}
An SPN is complete if for every sum node $v \in S$ and for every
$c_1, c_2 \in \children{v}$ it holds that $\scope{c_1} = \scope{c_2}$.
\end{definition}

\begin{definition}
An SPN is decomposable if for every product node $v \in P$ and for
every $c_1, c_2 \in \children{v}$, such that $c_1 \neq c_2$,
it holds that $\scope{c_1} \cap \scope{c_2} = \emptyset$.
\end{definition}

Generally, for an SPN that is not decomposable and complete,
$\Psi(\cdot)$ only represents an unnormalized distribution over
$X_1,\ldots,X_N$, due to the positive constraints on its weights, while
computing its normalization term is not typically tractable. A generative model
is said to possess \emph{tractable inference} if computing its normalized
probability function is tractable. Though general SPNs do not posses tractable
inference, limiting them to be decomposable and complete (D\&C) is a sufficient
condition for tractability, under-which computing the normalization term, i.e.
computing $\sum_{x_1,\ldots,x_N\in\{0,1\}}\Psi(x_1,\ldots,x_N)$ is equivalent
to
evaluating $\Psi(*, \ldots, *)$, and thus the normalized probability is given by
$P(X_1{=}x_1,\ldots,X_N{=}x_N){=}\cramped{\frac{\Psi(x_1,\ldots,x_N)}{\Psi(*,\ldots,*)}}$.
Also, not only is $\Psi(\cdot)$ a valid probability function, but for
any $v \in V$, $\Psi_v(\cdot)$ defines a valid distribution over~$\scope{v}$.
As shown by \citet{Peharz:2015tp}, simply normalizing the weights of each sum
node to sum to one ensures that $\Psi(x_1, \ldots, x_N)$ is already a normalized
probability function, with no need to compute a normalization factor, and
furthermore, this restriction does not affect the expressiveness of the model,
namely any SPN with unnormalized sum nodes could be converted to an SPN of same
size but with normalized sum nodes. Hence, for the remainder of the article we
will simply assume sum nodes have normalized weights.

It is important to understand why D\&C leads to tractability. The
decomposability condition ensures that the children of a product node do not
have shared variables, and because the product of distributions over different
sets of variables is also a normalized distribution, then a product node of
a decomposable SPN represents a normalized distribution as long as its children
represent normalized distributions. Similarly, the completeness condition
ensures that the children of sum nodes have the exact same scope, and because a
weighted average of distributions over the same set of variables, with
normalized sum weights, is also a normalized distribution over these variables,
then a sum node represents a normalized distribution if its children do as well.
Employing an induction argument, both conditions combined together guarantee that
every node in an SPN will represent a valid distribution.

An additional positive outcome of the D\&C condition is that not only is it
tractable to compute $P(X_1, \ldots, X_N)$, it is also tractable to compute any
of its marginals, e.g.~$P(X_1, \ldots, X_K)$ for $K < N$, by simply replacing
the values of a marginalized variable with the special value $*$, e.g.
$P(X_1{=}x_1, \ldots, X_K{=}x_K) = \Psi(x_1, \ldots, x_K, *,\ldots, *)$.
We call this last property \emph{tractable marginalization}, which is distinct
from the weaker property of tractable inference.

Lastly, learning an SPN model of a given structure is typically carried out
simply according to the Maximum Likelihood Principle, for which several methods
have been proposed, ranging from specialized Expectation Maximization algorithms
to gradient based methods, e.g.~simply performing Stochastic Gradient Ascent.

\gapbeforesection
\section{Sum-Product-Quotient Networks} \label{sec:model}
\gapaftersection

As discussed in sec.~\ref{sec:intro}, not all tractable distributions can be
represented by an SPN of a reasonable size, a limitation which stems from the
D\&C connectivity constraints imposed on the computational graphs of SPNs to
achieve tractable inference. In this section we describe an extension of SPNs,
under which we can relax these constraints and thus dramatically increase its
capacity to efficiently represent tractable distributions. At the heart of our
model is the introduction of a quotient node, i.e.~a node with two inputs, a
numerator and a denominator, that outputs their division. Quotient nodes can
have a natural interpretation as a conditional probability, i.e.~$P(A|B) =
\frac{P(A \cap B)}{P(B)}$. Hence, we call our model Sum-Product-Quotient
Networks, or SPQNs for short.

As with SPNs, not any computational graph made of sum, product and quotient
nodes results in a model possessing tractable inference. To ensure the
tractability of SPQNs, we introduce a set of restrictions generalizing the D\&C
conditions defined in sec.~\ref{sec:pre}. Formally, and in accordance with the
notations of sec.~\ref{sec:pre}, we denote by $Q$ the set of quotient nodes,
where $V \equiv {S \cup P \cup Q \cup I}$ is the set of all nodes, and for all
$v \in Q$ we denote its numerator and denominator nodes by $\qnu{v}$ and
$\qde{v}$, respectively. As we will shortly show, each node $v \in V$ of an SPQN
essentially represents a conditional distribution over the variables in its
scope, which give rise to a natural partition of the scope $\scope{v}$ into two
disjoint sets: (i)~\emph{conditioning scope}, denoted by $\qcond{v}$, and
(ii)~\emph{effective scope}, denoted by $\qeff{v}$~--~under this partition, for
tractable SPQNs, each node computes the conditional probability
$P_v(\qeff{v}|\qcond{v})$. Formally, the conditioning scope is defined as the
complement of the effective scope, i.e.
$\qcond{v} \equiv \scope{v} \setminus \qeff{v}$, while the effective scope is
defined the same as the general scope for all nodes except for quotient nodes,
namely, $\qeff{v} \equiv \{i\}$ for leaf nodes and
$\qeff{v} \equiv \cup_{c \in\children{v}} \qeff{c}$ for sum and product nodes. 
For quotient nodes we define $\qeff{v} \equiv \qeff{\qnu{v}} \setminus
\qeff{\qde{v}}$ following our intuition of quotient nodes as conditional
probabilities, e.g.~for $P(X_1|X_2,X_3) = \frac{P(X_1,X_2|X_3)}{P(X_2|X_3)}$ it
holds that $\qeff{v} = \{1\}$ and $\qcond{v} = \{2,3\}$, because we started with
the effective variables of the numerator $\qeff{\qnu{v}} = \{1,2\}$, from which
we subtracted the effective variables of the denominator $\qeff{\qde{v}}=\{2\}$.

With the above definitions in place, we are now ready to present our
generalization of the D\&C conditions for SPQNs. As discussed in
sec.~\ref{sec:pre}, the intuition behind the D\&C conditions is that they allow
for a rather basic way to combine the distributions defined by the children of
a given node, each over their respective scope, to form a valid distribution
over the scope of their parent node. In broad terms, we simply carry over the
same idea to SPQNs, but apply it on conditional distributions instead. For sum
and product nodes, this translates in essence to applying the D\&C
conditions with respect to the effective scope of a node instead of its general
scope, which we formalize as:
\begin{definition}\label{def:cond_comp}
    An SPQN is \emph{conditionally complete} if it is complete with
    respect to the effective scope, i.e.~for every sum node $v \in S$ and for
    every $c_1, c_2 \in \children{v}$, it holds that $\qeff{c_1} = \qeff{c_2}$.
\end{definition}
\begin{definition}\label{def:cond_decomp}
    An SPQN is \emph{conditionally decomposable} if for every product
    node $v \in P$:
    \begin{enumerate}[nosep]
        \item It is decomposable with respect to the effective scope, i.e.~for
              every $c_1, c_2 \in \children{v}$, such that $c_1 \neq c_2$, it
              holds that $\qeff{c_1} \cap \qeff{c_2} = \emptyset$.
        \item Its induced dependency graph over its children does not contain
              a cycle, where the directed graph is defined by the vertices
              $\children{v}$ and edges $\{ {c' \to c''} | c', c''
              \in \children{v}, \qeff{c'} \cap \qcond{c''} \neq \emptyset \}$.
    \end{enumerate}
\end{definition}
Under the conditional completeness condition, for every sum node $v \in S$, and
for any fixed values to the variables in its conditional scope $\qcond{v}$, we
can treat the conditional distributions of its children simply as
distributions over the variables in the effective scope. Because $v$ is
complete with respect to the effective scope, then following the same arguments
as in sec.~\ref{sec:pre}, $v$ represents a distribution as long as its children
do as well. The above logic can also be applied to product nodes under a more
restrictive form of conditional decomposability, where for every child $c \in
\children{v}$ it holds that $\qcond{c} {\subset} \qcond{v}$, under which the
variables in the conditional scope of each child node are fixed. However, under the
more general setting of conditional decomposability, there could be shared
variables between the conditional scope of one child $c_1 {\in} \children{v}$ and
the effective scope of another child $c_2 {\in} \children{v}$~--~in which case we
say that $c_1$ depends on $c_2$, as the probability of the effective scope of
$c_1$ is conditioned on the variables in the effective scope of $c_2$. By
representing all the dependencies between the children of $v$ as a directed
graph, then if each child represents a valid conditional distribution over its
scope and the graph is \emph{acyclic}, then it effectively defines a Bayesian
Network factorization to the conditional probability over the scope of $v$,
hence $v$ too is a valid conditional distribution.

At this point it is important to note how conditional D\&C are actually relaxed
versions of their ``unconditional'' counterparts. First, notice that when the
conditional scope is empty, i.e.~when the sub-graph rooted at $v$ contains only
sum and product nodes, or in other words this sub-graph is an SPN, then
conditional D\&C are equivalent to D\&C. Second, and more importantly, notice
that when the conditional scope is nonempty, conditional decomposability allows
taking the product of nodes with overlapping scopes, which is forbidden under
the stricter decomposability constraint. This entails that conditional D\&C
SPQNs allow for a richer set of structures than D\&C SPNs.

At last, to ensure the tractability of SPQNs we must also introduce a condition
on its quotient nodes, to which there is no equivalent in classical SPNs. The
following condition captures our motivation of a quotient node as a way to
compute conditional distributions by direct representation of their definition,
i.e.~that the denominator is a strictly positive marginal distribution of the
numerator:
\begin{definition}\label{def:cond_sound}
    An SPQN is \emph{conditionally sound} if for every quotient node
    $v \in Q$, it holds that $\Psi_{\qde{v}}(\cdot)$ is strictly positive, as
    well as a \emph{marginal} of $\Psi_{\qnu{v}}(\cdot)$, i.e.~that
    $\qcond{\qde{v}} \subset \qcond{\qnu{v}}$, $\qeff{\qde{v}} \subset 
    \qeff{\qnu{v}}$, and for all $\aaa \in \{0,1,*\}^N$ it holds that:
    \begin{align*}
        \Psi_{\qde{v}}(\aaa) =
            \sum_{\mathclap{\substack{
                    \z \in \{0,1,*\}^N \\
                    \forall i, i\not\in\qeff{v} \to z_i = a_i \\
                    \forall i, i\in\qeff{v} \to z_i \in \{0,1\}
                  }}}
                \Psi_{\qnu{v}}(\z)
    \end{align*}
    An SPQN is \emph{strongly} conditionally sound if in addition to the above,
    for $\z {\in} \{0,1,*\}^N$ such that $z_i {=} *$ if $i {\in} \qeff{v}$
    and otherwise $z_i {=} a_i$, it holds that $\Psi_{\qde{v}}(\aaa) {=}
    \Psi_{\qnu{v}}(\z)$.
    
\end{definition}

The definition of strong conditional soundness above is not required for
tractability~--~only the weaker conditional soundness~--~but does ensure
efficient sampling as discussed in sec.~\ref{sec:model:sampling}. We conclude by
formally proving that an SPQN that meets the above conditions, which will
henceforth be referred to as a \emph{tractable SPQN}, results in a tractable
generative model, as described by the following theorem
(see app.~\ref{app:proof:tractable_spqns} for proof):
\begin{theorem}\label{theorem:tractable_spqns}
    For any conditionally decomposable, conditionally complete, and
    conditionally sound SPQN over the random binary variables
    $X_1, \ldots, X_N$, for all $v \in V$, and any values of the variables found
    in $\qcond{v}$, it holds that $\Psi_v(\cdot)$ is a normalized probability
    function over $\qeff{v}$ conditioned on $\qcond{v}$.
\end{theorem}

Given an SPQN with a fixed structure that meets the tractability conditions of
theorem~\ref{theorem:tractable_spqns}, then its output is a differential
probability function of the data, and so we can learn its parameters simply by
maximizing the likelihood of the data through gradient ascent methods, as
commonly employed by both SPNs and other deep learning methods. Adjusting other
methods typically used to learn SPNs, e.g.~EM-type algorithms for parameter
learning and the various suggested structure learning algorithms, is deferred to
future works.

Though theorem~\ref{theorem:tractable_spqns} provides sufficient conditions for
SPQNs to be tractable, it is not prescriptive as to how exactly these models
must be structured. Specifically, while the conditionally decomposable and
conditionally complete conditions are quite simple to follow, it is generally
not clear how to adhere to the conditionally sound condition. We address this
in the next section.

\gapbeforesubsection
\subsection{Conditional Mixing Operator}\label{sec:model:cmo}
\gapaftersubsection

As discussed in the previous section, tractable SPQNs must comply with the
conditionally sound condition, and verifying that a given model adheres to
it is nontrivial. In this section, we suggest instead to follow a stricter
restriction that leads to a concrete construction of a tractable
SPQN. Specifically, we define a building block operator composed of sum,
product, and quotient nodes that guarantees the resulting model to be
tractable, which we call the \emph{Conditional Mixing Operator}:
\begin{definition}\label{def:cmo}
    The \emph{Conditional Mixing Operator}~(CMO) over non-negative
    matrices $A \in \R_{+}^{\gamma, \alpha}$ and $B \in \R_{+}^{\gamma, \beta}$,
    where $\alpha, \beta, \gamma \in \N$, $\beta > 0$, and parametrized by
    strictly positive weights $\w \in \R_{+}^\gamma$ such that
    $\sum_{i=1}^\gamma w_i = 1$, is defined as follows:
    \begin{align}\label{eq:cmo}
        \CMO(A, B; \w) = \frac{\sum\limits_{\mathclap{i=1}}^\gamma w_i \left(\prod\limits_{\mathclap{j=1}}^\alpha A_{ij}\right)\cdot \left(\prod\limits_{\mathclap{j=1}}^\beta B_{ij}\right)}{\sum\limits_{\mathclap{i=1}}^\gamma w_i \prod\limits_{\mathclap{j=1}}^\alpha A_{ij}}
    \end{align}
    In the context of SPQNs, a CMO node with children $a_{11},\ldots,
    a_{\gamma \alpha}, b_{11},\ldots,b_{\gamma \beta} \in V$ outputs
    $\CMO(A,B;\w)$, where $A_{ij} = \Psi_{a_{ij}}(\cdot),
    B_{ij} = \Psi_{b_{ij}}(\cdot)$.
    
\end{definition}
The motivation behind this construction is its connection to the conditional
probability of a mixture model. Notice that the numerator of eq.~\ref{eq:cmo}
essentially represents a mixture model with decomposable mixing components
divided into two sets according to $A$ and $B$, while the denominator represents
the marginalization over the variables relating to $B$.

The tractability of SPQNs composed of CMOs is ensured by the definition a
\emph{valid CMO node} as follows: 
\begin{definition}\label{def:valid_cmo}
    A CMO node with children $a_{11}, \ldots, a_{\gamma \alpha}, b_{11},\ldots,
    b_{\gamma \beta} \in V$ is said to be valid if the following conditions are
    met:
    \begin{enumerate}[nosep]
        \item The children of a CMO node are either valid CMO nodes themselves,
              or it holds that $\alpha = 0, \beta = 1, \gamma = 2$, and its
              children are exactly $\indc{x_i = 0}$ and $\indc{x_i = 1}$ for
              some $i \in [N]$.
        \item The internal sum nodes of the CMO are conditionally complete.
        \item The internal product nodes of the CMO, i.e.~the ones computing
              $\prod_{j=1}^\alpha A_{ij}$, $\prod_{j=1}^\beta B_{ij}$, and
              $\left(\prod_{j=1}^\alpha A_{ij}\right){\cdot}
               \left( \prod_{j=1}^\beta B_{ij}\right)$, are conditionally
              decomposable, and in the dependency graph of the top product
              node there are no arrows pointing from $B$ to $A$.
        \item $\forall i_1,i_2 \in [\gamma]$,\,\,
              $\qeff{\prod_{j=1}^\beta B_{i_1j}} = \qeff{\prod_{j=1}^\beta B_{i_2j}}$.
    \end{enumerate}
\end{definition}
We proceed to formalize our claim as follows:
\begin{proposition} \label{prop:valid_cmo_tractability}
    Any SPQN that is composed of valid CMO nodes is tractable. Moreover, it is
    strongly conditionally sound.
\end{proposition}

\begin{proof}[Proof Sketch]
    Since the internal sum and product nodes of a valid CMO are already
    conditionally D\&C, it is only left to show that it is also conditionally
    sound. This is achieved by an induction argument on the depth of an SPQN
    composed of valid CMOs, where we assume all nodes up to a given depth $d$
    are strictly positive, conditionally sound, and hence also represent valid
    distributions according to theorem~\ref{theorem:tractable_spqns}. By the
    assumption, the internal sum and product nodes of a depth $d+1$ valid CMO
    node also represent valid distributions, as they are already conditionally
    D\&C. Hence we can directly compute its marginalization over the variables
    in the effective scope of the B-type children, to conclude our proof of
    conditional soundness. Strong conditional soundness follows from conditional
    soundness and the definition of the CMO, since placing $*$ in all variables
    of the effective scope of the B-type children is equivalent to substituting
    their values with $1$'s. See our complete proof in
    app.~\ref{app:proof:valid_cmo_tractability}.
\end{proof}

Unlike conditional soundness, it is practical to validate that all CMO nodes
in a given SPQN are valid. Simply start at the root and recursively validate
that each of the children of a given node are valid, with the base case of CMO
nodes connected to one of the indicator nodes, as govern by the first condition
in def.~\ref{def:valid_cmo}. We then proceed to verifying that the internal
product and sum nodes follow the conditional D\&C constraints, by simply testing
their effective and conditional scopes according to def.~\ref{def:cond_comp} and
def.~\ref{def:cond_decomp}.

Though valid CMOs pave the way to tractable SPQNs, they raise the question of
what we have lost in the process. Indeed, conditional soundness allows for a
richer set of valid structures than valid CMOs, e.g.~they allow for the
distribution at the denominator and numerator of a quotient node to be defined
by completely different sub-graphs, unlike with CMOs that share children.
While we have yet to determine if there is a significant expressivity gap
between these two cases, an important property of an SPQN composed of valid CMOs
is that any D\&C SPN can be effectively represented by such a
model\footnote{An edge case of SPNs which demands a unique treatment is when
there exists a sum node which is connected to just one of $\indc{x_i=0}$ or
$\indc{X_i=1}$, but not both, while a valid CMO must have positive weights for
both indicator leaves. In this scenario we can instead arbitrarily approximate
the SPN, by approaching the zero weight $\epsilon \to 0$.}, hence this
restriction is at least as expressive as any D\&C SPN. In
sec.~\ref{sec:analysis} we show that they are in fact significantly more
expressive than SPNs.

\gapbeforesubsection
\subsection{The Generative Process of SPQNs}\label{sec:model:sampling}
\gapaftersubsection

In prior sections we have presented our SPQN model, and showed that it can be
tractable under simple conditions, and more importantly that any of its internal
nodes represent a conditional distribution over its scope. In this section we
leverage these relations to describe the generative process of SPQNs, showing
sampling from an SPQN is just as efficient as inference, under the strongly
conditional soundness constraint. The ability to efficiently draw samples from a
probabilistic model is a highly desirable trait with many applications, e.g.
completing missing values, and introspection of the learned models.

Sampling from a tractable SPQN model follow the same general steps as sampling
from a D\&C SPN. We begin at the root node of the graph, and then stochastically
traverse the nodes according to parameters of the model, until we
reach the indicator nodes, each representing the sampled value for its
respective random variable. In SPNs, traversal follow two simple rules: (i) if
we encounter a product node, then because it is decomposable, each child is a
distribution over separate sets of variables, hence we can recursively sample
from each child separately; (ii) if we encounter a sum node, then we sample one
of its children according to the categorical distribution defined by their
respective weights.
Given that SPQNs are extensions of SPNs, their generative process can be seen as
simply a generalization of the traversal rules of the SPNs. However, their
distinctive property of having nodes which represent conditional distributions,
calls for some adjustments. Namely, it is not only required to traverse
the graph, but also to keep track of the values that have already been sampled
so far in the process, and then pass it along to nodes which depend on it.

\begin{algorithm}
    \caption{Sampling procedure for SPQNs. Accepts as input a node $v \in V$,
    and a partial sample $\s \in \{0,1,*\}^N$, where $*$ denotes missing
    values.}
    \small
    \label{alg:sampling}
    \begin{algorithmic}[1]
        \Function{SampleSPQN}{$v, \s$}
            \If{$v \in Q$}
                \State $\s \gets $ \Call{SampleSPQN}{$\qnu{v}, \s$}
            \ElsIf{$v \in P$}
                \State $children \gets$ \Call{TopologicalSort}{$\children{v}$}
                \ForAll{$c \in children$}
                    \State \algorithmicif\ $\forall i \in \qeff{c}, s_i \neq *$
                            \algorithmicthen\ skip iteration
                    \State $\s \gets $ \Call{SampleSPQN}{$c, \s$}
                \EndFor
            \ElsIf{$v \in S$}
                \State $\w \gets $ \Call{GetWeights}{$v$}
                \ForAll{$c \in \children{v}$}
                    \State $w_c \gets w_c {\cdot} \Psi_c(\s)$
                \EndFor
                \State $\w \gets \nicefrac{\w}{\sum_c w_c}$
                \State $c \sim Cat(\children{v}, \w)$
                \State $\s \gets $ \Call{SampleSPQN}{$c, \s$}
            \ElsIf{$\exists i \in \{1,\ldots,N\}, v \equiv \indc{x_i = a}$}
                \State $s_i \gets a$
            \EndIf
        \State \Return $\s$
        \EndFunction
    \end{algorithmic}
\end{algorithm}

The above reasoning brings us to algo.~\ref{alg:sampling}, which receives as
input a starting root node $v \in V$, and a partial sample
$\s \in \{0, 1, *\}^N$, where $s_i = *$ denotes values which have yet to be
sampled. Typically, the first call to algo.~\ref{alg:sampling} will be with the
root $v \in V$ and $\s = (*,\ldots,*)$, i.e.~sampling a complete instance
$X_1,\ldots,X_N \sim P(X_1, \ldots, X_N)$, but often times it is useful to also
be able to sample from the conditional distribution\footnote{Exactly sampling
from a conditional distribution is possible only if it respects the dependencies
induced by the model on the input variables.}, e.g.~$X_1,\ldots,X_K \sim
P(X_1,\ldots,X_K | X_{K+1}{=}s_{k+1},\ldots,X_N{=}s_N)$ by calling with
$\s = (*,\ldots,*,s_{k+1},\ldots,s_N)$. The inner-workings of
algo.~\ref{alg:sampling} follow the traversal workflow of SPNs as described
above, with the following adjustments: (i)~For quotient nodes, we directly
traverse to its numerator child, as the denominator only serve as a
normalization factor. (ii)~For product nodes, though the effective scopes of the
children are disjoint sets and could be processed separately as with SPNs, the
dependencies induced by the conditional scopes of each child require sampling
according to the topological order of the dependencies graph. Additionally,
there is the possibility that the effective scope of some child nodes have
already been sampled, in which case we simply skip it. (iii)~For sum nodes, the
probability of sampling each child is no longer given just by its weights, but
also by the marginal probability of the already sampled variables given by $\s$,
namely if $Q \equiv \{i \in [N]| s_i {\neq} *\}$ denotes the set of sampled
variables then we can factor the conditional distribution of the sum node
$v \in S$, i.e.~$P_v(\qeff{v} {\setminus} Q | Q)$, as the following expression:
\begin{multline*}
        \smashoperator[lr]{\sum_{c \in \children{v}}} P_c(\qeff{c} {\setminus} Q | Q)
            \frac{\cramped{w_c \cdot P_c(\qeff{c} {\cap} Q | \qcond{c})}}{\smashoperator[lr]{\sum_{c' \in \children{v}}} w_{c'} \cdot P_{c'}(\qeff{c'} {\cap Q} | \qcond{c'})}
\end{multline*}
where $P_c(\qeff{c} \cap Q | \qcond{c})$ can be computed by $\Psi_c(\s)$
according to strong conditional soundness, and thus the probability of sampling
the child $c$ is $\propto w_c \cdot \Psi_c(\s)$.

Finally, regarding the complexity of the sampling algorithm, traversing
the computational graph is linear in the number of nodes, and while computing
$\Psi_v(\cdot)$ when sampling from sum nodes could result in an $O(|V|^2)$
runtime, in practice we could reuse prior computations to reduce it to just
$O(|V|)$. In this analysis, we do not take into account the topological sort
applied to the children of the product nodes, as this is a one time operation
that is not required for every sampling. In conclusion, sampling from a
tractable SPQN that is also strongly conditionally sound, e.g.~by composition of
valid CMOs, is just as efficient as with SPNs.

\gapbeforesection
\section{Analysis of Expressive Efficiency} \label{sec:analysis}
\gapaftersection

In sec. \ref{sec:model}, we have shown that tractable SPQNs extend D\&C SPNs,
and can thus efficiently replicate any tractable distribution that D\&C SPNs can 
realize. In this section, we will show a simple tractable distribution which
SPQNs can realize, but D\&C SPNs cannot, unless their size is exponential in the 
length of their input, where the size of an SPQN (or SPN) is defined as the
number of its internal nodes. More specifically, we show that tractable SPQNs
can represent a strictly positive distribution of sampling an undirected
triangle-free graph on $M$ vertices, where each edge is represented by a random
binary number, while D\&C SPNs of polynomial size cannot represent, or even
approximate, such distributions.

First, let us formally define a strictly positive distribution over
triangle-free undirected graphs on $M$ vertices. We define the binary random
variables $\EEE \equiv \left\{ E_{ij} | 1 \leq i < j \leq M \right\}$,
such that if $E_{ij} = 1$, then the edge $\{ i, j\} $ is part of the graph, and
not otherwise, and denote by $N \equiv \left| \EEE \right| = \binom{M}{2}$ the
number of variables. For a given graph, we say it contains a triangle if and
only if there are three vertices in the graph such that between any two of them
there is an edge, i.e.~there exists $i_1 < i_2 < i_3$ such that
$(E_{i_1 i_2} {=} 1) \wedge (E_{i_2 i_3} {=} 1) \wedge (E_{i_1 i_3} {=} 1)$.
Finally, we say that a probability function $d\left(\mathbf{E}\right)$ on the
edges $\EEE$ is a \emph{strictly positive distribution on triangle-free graphs}
if it holds that $d(\EEE) > 0$ if and only if $\forall i_1 < i_2 < i_3,
(E_{i_1 i_2} {=} 0) \vee (E_{i_2 i_3} {=} 0) \vee (E_{i_1 i_3} {=} 0)$.

The above definition falsely appears to lead to an efficient realization through SPNs
of a strictly positive distribution on triangle-free
graphs: simply define a node for each potential triangle, such that it is
positive only if it is legal, i.e.~at least one of its edges is not part of the
graph, and then take the product of all such nodes to guarantee all triangles are 
legal. More specifically, we can define a sum node for each triplet
$(E_{i_1 i_2}, E_{i_2 i_3}, E_{i_1 i_3})$, for which there are $\binom{M}{3}$
combinations, such that each sum node is equal to $(\indc{E_{i_1 i_2}{\shorteq}0} {+}
\indc{E_{i_2 i_3}{\shorteq}0} {+} \indc{E_{i_1 i_3}{\shorteq}0})$, and then take the product of
all of these sum nodes and modify their weights such that they output a
normalized probability function. However, this SPN is not
D\&C, because each sum node does not meet the completeness condition
as its children have different scopes, e.g.~$\scope{\indc{E_{i_1 i_2}{\shorteq}0}} \neq
\scope{\indc{E_{i_2 i_3}{\shorteq}0}}$, and because the product node over all sum nodes
does not meet the decomposability condition, as each edge $E_{i_1 i_2}$ is
present in multiple triplets, i.e.~multiple child nodes, resulting in non-disjoint scopes.
Because it is not D\&C,
computing its normalization factor in practice is intractable. More generally, we can show
that any D\&C SPN approximating a strictly positive distribution on
triangle-free graphs must be exponentially large:

\begin{theorem} \label{thm:spn_not_efficient}
Let $d\left(\mathbf{E}\right)$ be a strictly positive distribution on
triangle-free graphs of $M$ vertices. Suppose that $d(\EEE)$ can be approximated 
arbitrarily well by D\textsl{\&}C SPNs of size $\leq s$. Then
$s \geq 2^{\Omega(M)}$.
\end{theorem}

\begin{proof}[Proof Sketch]
We have modified the proof of a similar theorem by \citet{Martens:2014tr}, which 
showed that a D\&C SPN that can approximate arbitrarily well the probability
function of the uniform distribution on the spanning trees of the complete
graph, must be of size $\geq2^{\Omega\left(M\right)}$. See
app.~\ref{app:proof:spn_not_efficient} for the complete modification of that
proof to our case.
\end{proof}

In contrast to D\&C SPNs, tractable SPQNs can efficiently realize at least some
strictly positive distributions on triangle-free graphs, with size at most
polynomial in $M$. In the case of SPQNs built on CMOs, exact realization is
replaced by arbitrarily good approximation, without any size increase. This is
formalized by the following theorem:
\begin{theorem} \label{thm:spqn_are_efficient}
There exists a tractable SPQN exactly realizing a probability function
$d(\EEE)$, such that $d(\EEE)$ is a strictly positive distribution on
triangle-free graphs of $M$ vertices, where the size of the SPQN is
$O(M^4)$. In the case of SPQNs composed strictly of CMOs, instead of exact
realization, they can approximate said distribution arbitrarily well with size
$O(M^4)$.
\end{theorem}

\begin{proof}[Proof Sketch]
Taking inspiration from the failed attempt to realize such a distribution via
D\&C SPNs, let us now construct a tractable SPQN which does realize such a
distribution efficiently. As before, we begin by examining all potential
triangles, but instead of directly modelling the constraints individually, we
group them by their largest edge (according to lexical ordering). For each edge
and its respective group of triangles, we can define the conditional probability 
of that edge conditioned on all other edges participating in these triangles,
such that the conditional probability is non-zero only if triangles which
include this edge are not all part of that graph. For edges that are not part of 
any triangle for which they are the largest edge, we simply define a sum node
which represent an equal probability for including the edge or not. Finally, we
can simply take the product of all conditional distributions of each edge,
giving  rise to a normalized probability function over all edges $\EEE$, which
is non-zero if and only if the edges in $\EEE$ represent a triangle free graph.
See app.~\ref{app:proof:spqn_are_efficient} for our complete proof.
\end{proof}

To conclude, we have shown that tractable SPQNs, as well as ones
composed of valid CMOs, are exponentially efficient with respect to D\&C SPNs.

\gapbeforesection
\section{Discussion and Related Works} \label{sec:discussion}
\gapaftersection

In this work we address the limited expressive efficiency of SPNs, which
\citet{Martens:2014tr} have proven to be incapable of approximating even simple
tractable distributions, unless their size is exponential in the number of
variables. To mitigate this limitation of SPNs, we have presented a novel
extension to SPNs which we call Sum-Product-Quotient Networks, or SPQNs for
short. SPQNs introduce a new node type that computes the quotient of its two
inputs, which in part enabled us to relax the strict structural conditions that
are commonly used to ensure the tractability of SPNs. By requiring less strict
conditions for tractability, we have proven that SPQNs are a strict superset of
SPNs, and moreover that SPQNs are exponentially more expressive efficient than 
SPQNs.

There is a vast literature on analyzing the expressivity of arithmetic
circuits~(ACs)~\citep{Shpilka:2010cu,expressive_power,inductive_bias,
Cohen:0ZJHmEow, Levine:2017wt}, and more particularly of
SPNs~\citep{Delalleau:2011vh, Martens:2014tr}. Notable amongst those is the work
of \citet{sharir2017expressive}, where they compared the expressive efficiency
of Convolutional ACs~(ConvACs) having no overlapping receptive
fields, which are equivalent to a sub-class of SPNs following a tree-structure
partitioning of scopes, against a ConvAC with overlaps, which have no equivalent
D\&C SPN. They have found that simply introducing overlaps, i.e.~breaking the
decomposability condition, had the effect of exponentially increasing the
expressive efficiency of the model. A closer examination of their overlapping
ConvAC reveals that it shares the same construct as the numerator of our
CMOs nodes, but without the denominator, and thus their results could be
trivially adapted to SPQNs following a similar architecture. This entails that
not only are there some distributions which SPQNs can represent efficiently
that SPNs cannot, as we have showed in sec.~\ref{sec:analysis}, but that
almost all distributions realized by SPQNs cannot be realized by tree-like
SPNs\footnote{Not to be confused with Sum-Product Trees~\citep{Peharz:2015tp}
that are a far more restricted sub-class of SPNs, in which every sum and
product nodes have just a single parent, as opposed to limiting just the sum
nodes to have a single parent as in non-overlapping ConvACs.}, known also as
Latent Tree Models~\citep{Mourad:2013kz}, unless they are of exponential size.
Nevertheless, it is important to stress the importance of our own results,
which separate between SPQNs and D\&C SPNs of any conceivable structure, and
not just a small sub-class of SPNs.

Recently, \citet{Telgarsky:2017tu} has examined the relations between neural
networks and rational functions, i.e.~quotient of two polynomials, as well as a
model he called \emph{rational networks}, which is a neural network with
activation functions limited to only rational functions. He found that a new
neural network with ReLU activations could be approximated arbitrarily well by a
similarly size rational network, and that the reverse is true as well. 
Though this might seem to suggest that SPQNs could be on par with neural
networks, \citet{Hoover:1990jj} proved that any computable
function can be realized by ACs~--~hence the power of
SPQNs is not due to quotient nodes, but rather their richer structure.

In the broader literature on ACs~\citep{Shpilka:2010cu},
the proposal of introducing a quotient node has been previously considered and
deemed unnecessary. Their argument is based on the proof that in circuits which
compute polynomial functions all quotient nodes could be replaced by just a
single negation node, or in other words, that a quotient node does not add any
power to ACs. Despite this negative outcome, it does not apply
to our case on two accounts: (i)~It assumes the output of the circuit is
identically a polynomial function instead of a rational function,
and since the proof itself relies on the structural properties of the polynomial,
namely its degree and homogenous decomposition, it cannot be adapted to our case.
(ii)~It does not apply to monotone ACs, where the weights are restricted
to be non-negative, as is the case of SPNs, where negation is not allowed.
In this context, it was proven that even a single negation gate leads to
exponential separation from monotone circuits, and while quotient nodes could
be replaced by negation, the reverse is not generally true, hence this last
result does not trivialize our own. Overall, given our results, it might
suggest that the role of quotient nodes should be reexamined for ACs.

While we prove that our SPQN model is exponentially more expressive than D\&C
SPNs, this increase in expressive efficiency does not come without a cost. One
of the great advantages of SPNs is that they not only possess tractable
inference, but also tractable marginalization~(see sec.~\ref{sec:pre}).
This uncommon ability amongst generative models has many uses,
e.g.~for missing data~\citep{Rubin:1976gv,tmm}. However, once we
relax decomposability to conditional decomposability, it means that SPQNs
effectively induce a partial ordering on the input variables, which limit
tractable marginalization only to the subsets of the variables that agree with
the ordering. While there appear to be fewer tasks which benefit
significantly from tractable marginalization compared to just tractable
inference, in the cases in which it is required, SPNs even with their limited
expressivity still have an advantage over SPQNs. This is a limitation that we
aim to address in future works, as we detailed below. Additionally,
\citet{Martens:2014tr} have shown that under mild assumptions D\&C is not only
sufficient but also necessary for tractable marginalization, which entails that
any possible relaxation to D\&C would result in losing general tractable
marginalization, hence it is not specific to the case of SPQNs.

Other recent works on tractable generative models have mainly focused on the
family of autoregressive models that are based on neural networks, most notable
amongst are NADE~\citep{JMLR:v17:16-272}, PixelRNN~\citep{vandenOord:2016um},
and PixelCNN~\citep{vandenOord:2016tk}. Despite the significant differences
between the underlying operations of SPQNs and these models, there are also some
similarities and shared concepts. Specifically, both our model and theirs are
based on inducing a partial ordering on the input variables, and modelling the
conditional probabilities between subsets of them, with the main difference as
to how these probabilities are represented. While they employ neural networks as
a black box to model them, we leverage interpretable SPNs to compose
conditional distributions in a hierarchy. We conjecture that the embedded
hierarchy of conditional distributions used in our model leads to an advantage
in  terms of its expressive capacity, while, in addition, the interpretable
nature of the inner-workings of our model has many real-world applications.

Lastly, we conducted preliminary experiments demonstrating the practical
advantages of SPQNs over SPNs in app.~\ref{app:exp}. Nevertheless, it still
remains to be verified that their superior expressive power translates to
real-world applications~--~a task we aim to tackle in future works. Beyond
that, SPQNs give rise to many straightforward extensions:
\begin{enumerate}[nosep]
    \item \textbf{Generative Classifier:} SPQNs could be naturally extended to
          represent a distribution conditioned on a given class, i.e.~$P(X|Y)$,
          especially suitable for semi-supervised learning, and for
          classification under missing data.
    \item \textbf{Tractable Marginalization:} despite that marginalization is
          not naturally supported by SPQNs, they do induce normalized
          distributions over any subset of its input variables, which are not
          generally consistent with each other. Joint training of SPQNs on
          random subsets of its variables, could be sufficient for ensuring the
          consistency of the induced marginal distributions.
    \item \textbf{Convolutional SPQNs:} our model has a natural formulation as a
          ConvNet-like generative model following the theoretical
          architecture of ConvACs with overlaps~\citep{sharir2017expressive}.
          The unparalleled success of ConvNets
          with the theoretical advantages of SPQNs has potential for rivalling
          neural tractable generative models.
\end{enumerate}
This paper has demonstrated the theoretical viability of Sum-Product-Quotient
Networks, suggesting a promising outlook for the above research directions.

\newcommand{\acknowledgments}{This work is supported by Intel grant ICRI-CI \#9-2012-6133,
by ISF Center grant 1790/12 and by the European Research Council (TheoryDL project).}
\ifdefined\CAMREADY
	\subsubsection*{Acknowledgments}
	\vspace{-2mm}
	\acknowledgments
\fi

\subsubsection*{References}
\small{
\bibliographystyle{plainnat}
\bibliography{refs.bib}
}

\clearpage
\appendix

\section{Deferred Proofs}\label{app:proofs}

This section contains proofs which were deferred to it from the body of the
article.

\subsection{Proof of Theorem \ref{theorem:tractable_spqns}}
\label{app:proof:tractable_spqns}

We will prove the theorem using induction based on depth of the circuit rooted
at a node $v \in V$, i.e.~the maximal length of a path connecting a leaf to $v$.
Given that $\Psi_v(\cdot)$ is a non-negative function, it is sufficient to show
it is normalized, i.e.~that for any fixed values of the variables in
$\qcond{v}$, denoted by $\bb \in \{0,1,*\}^N$ where $b_i = *$ if
$i \not\in \qcond{v}$, the following holds:
\begin{align*}
    \sum_{\mathclap{\substack{
        \x \in \{0,1,*\}^N \\
        \forall i \not\in \qeff{v}, x_i = b_i \\
        \forall i \in \qeff{v}, x_i \in \{0,1\}
        }}} \Psi_v(\x) = 1
\end{align*}

For the base case of the induction of depth-1 SPQNs, which means $v \in I$ must
be an indicator node, i.e.~$v = \indc{x_i = a}$ for some $i \in [N]$ and
$a \in \{0,1\}$, then $\qeff{v} = \{i\}$ and $\qcond{v} = \emptyset$, and so
summing over all possible values of $x_i$ is equal to
$\indc{x_i = a}(0) + \indc{x_i = a}(1) = 1$ meeting the normalization condition.
Let us now assume that our induction assumption holds for all circuit of depth
$d \geq 1$, and prove it also holds for $d+1$. Since any SPQNs of depth $d+1$ is
greater than 1, then the root node must either be a sum, product or quotient
node, and not an indicator node. Additionally, for the root $v \in V$ of such a
circuit, because each of its child nodes can be viewed as a depth-$d$
sub-circuit, then according to the induction assumption it represents a
normalized probability function over the variables in $\qeff{v}$ for any fixed
values of the variables in $\qcond{v}$. Next we will use this property to show
that for any possible node type, $v$ represent a normalized probability
function.

if $v \in Q$ is a quotient node, then according to conditional soundness then
$Psi_{\qde{v}}(\cdot)$ is a strictly positive function and hence the output of
the quotient operation is well defined. Additionally, the conditional soundness
also entails that $\Psi_{\qde{v}}(\cdot)$ is a marginal conditional distribution
of $\Psi_{\qnu{v}}(\cdot)$, and specifically, that summing
$\Psi_{\qnu{v}}(\cdot)$ over all the possible values of the variables in
$\qeff{v}$ equals to $\Psi_{\qde{v}}(\cdot)$, and thus:
\begin{align*}
    \smashoperator[lr]{\sum_{\substack{
        \x \in \{0,1,*\}^N \\
        \forall i \not\in \qeff{v}, x_i = b_i \\
        \forall i \in \qeff{v}, x_i \in \{0,1\}
        }}} \Psi_v(\x)
    &={\sum_{\substack{
        \x \in \{0,1,*\}^N \\
        \forall i \not\in \qeff{v}, x_i = b_i \\
        \forall i \in \qeff{v}, x_i \in \{0,1\}
        }}} \frac{\Psi_{\qnu{v}}(\x)}{\Psi_{\qde{v}}(\x)} \\
    &=\frac{1}{\Psi_{\qde{v}}(\bb)} \cdot
       \smashoperator[r]{\sum_{\substack{
        \x \in \{0,1,*\}^N \\
        \forall i \not\in \qeff{v}, x_i = b_i \\
        \forall i \in \qeff{v}, x_i \in \{0,1\}
        }}} \Psi_{\qnu{v}}(\x) \\
    &= \frac{1}{\Psi_{\qde{v}}(\bb)} \cdot \Psi_{\qde{v}}(\bb) = 1
\end{align*}
where we have used the fact that changing the values of the coordinates of $\x$
for $i \in \qeff{v}$ do not affect the value of $\Psi_{\qde{v}}(\x)$ as
$\qeff{\qde{v}} \cap \qeff{v} = \emptyset$, in combination with the relationship
between the sum over $\Psi_{\qnu{v}}(\cdot)$ and $\Psi_{\qde{v}}(\cdot)$.

If $v \in S$ is a sum node, then according to conditional completeness the
effective scopes of its child nodes are identical to its own effective scope.
This also entails that $\qcond{c} \subset \qcond{v}$ for any $c \in\children{v}$
because that $\qcond{c}$ is the complement of $\qeff{c}$ with respect to
$\scope{c}$. We can also assume without losing our generality that
$\qcond{c} = \qcond{v}$, as variables outside of $\qcond{c}$ do not affect the
output of $\Psi_c(\cdot)$ regardless of their value. Given the last assumption
and the induction assumption, all the children of $v$ represent conditional
distributions over the same set of variables, and because the weights of $v$ are
normalized to sum to one, then:
\begin{align*}
      \smashoperator[lr]{\sum_{\substack{
        \x \in \{0,1,*\}^N \\
        \forall i \not\in \qeff{v}, x_i = b_i \\
        \forall i \in \qeff{v}, x_i \in \{0,1\}
        }}} \Psi_v(\x)
      &={\sum_{\substack{
        \x \in \{0,1,*\}^N \\
        \forall i \not\in \qeff{v}, x_i = b_i \\
        \forall i \in \qeff{v}, x_i \in \{0,1\}
        }}} \smashoperator[lr]{\sum_{c \in \children{v}}} w_c \cdot \Psi_c(\x) \\
      &=\smashoperator[l]{\sum_{c \in \children{v}}} w_c \cdot
      \overbrace{\smashoperator[r]{\sum_{\substack{
        \x \in \{0,1,*\}^N \\
        \forall i \not\in \qeff{v}, x_i = b_i \\
        \forall i \in \qeff{v}, x_i \in \{0,1\}
        }}} \Psi_c(\x)}^{=1}  \\
        &= \smashoperator[l]{\sum_{c \in \children{v}}} w_c = 1
\end{align*}
where the inner sum equals to $1$ due to the normalization of the child nodes.

Finally, we will consider the case that $v \in P$ is a product node.
Recall that conditional decomposability means that the effective scopes of each
child of $v$ are disjoint sets, and that the directed dependency graph formed
by the children of $v$ is an acyclic graph. To prove this case, we will use a
secondary induction over the number of children of $v$. In the base case of $v$
having just a single child $\children{v} = \{c\}$, it holds that
$\Psi_v(\cdot) = \Psi_c(\cdot)$, and thus it is a normalized probability
function due to the primary induction assumption. Let us assume that out
secondary induction assumption holds for $v$ with $t$ children, and prove it
also holds for $t+1$ children. Let $\bar{c} \in \children{v}$ be child of $v$
that is a sink node in the induced dependency graph, i.e.~that none of the
variables in its effective scope are part of the conditional scope of another
child, hence the following holds:
\begin{align*}
      \smashoperator[r]{\sum_{\substack{
        \x \in \{0,1,*\}^N \\
        \forall i \not\in \qeff{v}, x_i = b_i \\
        \forall i \in \qeff{v}, x_i \in \{\!0,1\!\}
        }}} \Psi_v(\x)
      &=\smashoperator[r]{\sum_{\substack{
        \x \in \{0,1,*\}^N \\
        \forall i \not\in \qeff{v}, x_i = b_i \\
        \forall i \in \qeff{v}, x_i \in \{\!0,1\!\}
        }}} \Psi_{\bar{c}}(\x) \left(\smashoperator[r]{\prod_{\substack{c \in \children{v} \\ c \neq \bar{c}}}} \Psi_c(\x)\right) \\
      &\overset{(1)}{=}\smashoperator[l]{\sum_{\substack{
        \x \in \{0,1,*\}^N \\
        \forall i \not\in \qeff{v}, x_i = b_i \\
        \forall i \in \qeff{v}{\setminus} \qeff{\bar{c}}, x_i \in \{\!0,1\!\} \\
        \forall i \in \qeff{\bar{c}}, x_i = *
        }}} \!\!\!\!\left(\!\!\smashoperator[r]{\prod_{\substack{c \in \children{v} \\ c \neq \bar{c}}}} \Psi_c(\x)\!\!\right)
        \!\overbrace{\smashoperator[r]{\sum_{\substack{
        \z \in \{0,1,*\}^N \\
        \forall i \not\in \qeff{\bar{c}}, z_i = x_i \\
        \forall i \in \qeff{\bar{c}}, z_i \in \{0,1\}
        }}}
        \Psi_{\bar{c}}(\z)}^{=1} \\
      &\overset{(2)}{=}\smashoperator[l]{\sum_{\substack{
        \x \in \{0,1,*\}^N \\
        \forall i \not\in \qeff{v}, x_i = b_i \\
        \forall i \in \qeff{v}{\setminus} \qeff{\bar{c}}, x_i \in \{\!0,1\!\} \\
        \forall i \in \qeff{\bar{c}}, x_i = *
        }}} \!\!\!\!\left(\!\!\smashoperator[r]{\prod_{\substack{c \in \children{v} \\ c \neq \bar{c}}}} \Psi_c(\x)\!\!\right) = 1
\end{align*}
where the equality marked by $(1)$ is due to decomposing the sum into two nested
sums, one where we iterate over the different values of $\x$ just over the
coordinates matching the variables in the effective scope of $v$ that are not in
the effective scope of $\bar{c}$ and the second nested sum over the remaining
coordinates of the effective sum. Because the inner sum affects only the
variables in $\Psi_{\bar{c}}(\cdot)$ we can extract all over nodes out of it,
this is because of our assumption that $\bar{c}$ is a sink node and hence
$\qeff{\bar{c}}$ is not part of the scopes of the other children, in addition to
the fact that the effective scopes are disjoint sets. The equality marked by
$(2)$ is because $\Psi_{\bar{c}})(\cdot)$ is a normalized probability function
according to our primary induction assumption, hence the inner sum equals to
one. The final equality is due to our secondary induction assumption, as there
are only $t$ child nodes left and thus that sum also equals to one. This
concludes the proof for both the secondary and the primary induction assumption.

\subsection{Proof of Proposition \ref{prop:valid_cmo_tractability}}
\label{app:proof:valid_cmo_tractability}

By the second and third conditions in def.~\ref{def:valid_cmo}, all product
and sum nodes in an SPQN composed of valid CMOs must be conditionally D\&C,
and thus, according to theorem.~\ref{theorem:tractable_spqns}, we only need
to prove that it is conditionally sound for it to be tractable. We employ
induction on the depth of the SPQN rooted at $v \in V$ with the assumption that
all SPQNs up to depth $d$ that are composed of valid CMO nodes are strongly
conditionally sound, hence also valid distributions, strictly positive
functions, and that for all $\z \in \{0,1,*\}^N$ such that $z_i = *$ if
$i \in \qeff{v}$ it holds that $\Psi_v(\z) = 1$.

We begin with the base case of a CMO node connected to the two indicator leaf
nodes $\indc{x_i=0}$ and $\indc{x_i=1}$ for some $i \in [N]$, which according to
def.~\ref{def:valid_cmo} is the only valid CMO node that is connected to the
leaves. Under this case the output of the CMO node is equal to a single sum
node computing $w_1 \indc{x_i = 0} + w_2 \indc{x_i = 1}$, where $\w \in \R^2$ is
strictly positive. Since the output is simply a single sum node over
indicators of the same variable, it immediately follows that it is
conditionally decomposable, complete and sound. Additionally, since $\w$ is
strictly positive, then the output of the node is also strictly positive for
any value of $x_i$. Finally, when setting $x_i = *$ the output equals to
$w_1 \cdot 1 + w_2 \cdot 1 = 1$.

Let $v$ denote the root CMO node of an SPQN of depth $d+1$. Without losing our
generality, we can assume that $\alpha = \beta = 1$ (see def.~\ref{def:cmo})
with children $a_1,\ldots,a_\gamma,b_1,\ldots,b_\gamma \in V$, otherwise we can
substitute each of the products, $\prod_{j=1}^\alpha A_{ij}$ and
$\prod_{j=1}^\beta B_{ij}$, with an auxiliary valid CMO node that computes
just the product, i.e.~with no A-type children, which is trivially
conditionally sound. Since we assume all the children of $v$ represents
strictly positive functions, and since the output of $v$ is composed of
products and weighted sums with positive weights, then the output of $v$ is
also strictly positive. According to def.~\ref{def:valid_cmo}, the internal
sum and product nodes of $v$ are conditionally D\&C, and thus their
respective rooted sub-SPQNs are tractable by the induction assumption, which
means they represent valid distributions. Additionally,
def.~\ref{def:valid_cmo} also entails that the effective scopes of each of
$b_1,\ldots,b_\gamma$ are equal to $\qeff{v}$, and do not appear in the
conditional scopes of $a_1,\ldots,a_\gamma$. Now, for any $\aaa \in \{0,1,*\}^N$
the following holds:
\begin{align*}
    \sum_{\mathclap{\substack{
            \z \in \{0,1,*\}^N \\
            \forall i \not\in \qeff{v}, z_i = a_i \\
            \forall i \in \qeff{v}, z_i \in \{0,1\}
          }}} \Psi_{\qnu{v}}(\z)
    &= \smashoperator[l]{\sum_{\substack{
            \z \in \{0,1,*\}^N \\
            \forall i \not\in \qeff{v}, z_i = a_i \\
            \forall i \in \qeff{v}, z_i \in \{0,1\}
          }}} \sum_{i=1}^\gamma \Psi_{a_i}(\z) \Psi_{b_i}(\z) \\
    &\overset{(1)}{=}\sum_{i=1}^\gamma \Psi_{a_i}(\aaa)
    \overbrace{\smashoperator[lr]{\sum_{\substack{
                \z \in \{0,1,*\}^N \\
                \forall i \not\in \qeff{v}, z_i = a_i \\
                \forall i \in \qeff{v}, z_i \in \{0,1\}
              }}} \Psi_{b_i}(\z)}^{=1} \overset{(2)}{=} \Psi_{\qde{v}}(\aaa)   
\end{align*}
where the equality marked by $(1)$ is because the nodes of $a_i$ are not
affected by the changing coordinates specified by $\qeff{v}$, while the equality 
marked by $(2)$ follows from our induction assumption that the children
$b_1,\ldots,b_\gamma$ already represent normalized probability functions, and
thus summing over them equals to one.This proves that the denominator is a
marginal of the numerator, which prove that the SPQN rooted at $v$ is
conditionally sound. To prove that it is also strongly conditionally sound,
we simply notice that for any $\z \in \{0,1,*\}$ such that $z_i = *$ it holds
that $\Psi_{b_i}(\z) = 1$ based on our induction assumption as
$\qeff{b_i} = \qeff{v}$, and thus:
\begin{align*}
    \Psi_{\qnu{v}}(\z) = \sum_{i=1}^\gamma \overbrace{\Psi_{a_i}(\z)}^{=1} \Psi_{b_i}(\z) = \Psi_{\qde{v}}(\z)
\end{align*}
which proves that the SPQN rooted at $v$ is strongly conditionally sound.
Additionally from the conditionally sound property we have just proven, it thus
follow that
\begin{align*}
    \Psi_v(\z) = \frac{\Psi_{\qnu{v}}(\z)}{\Psi_{\qde{v}}(\z)} 
    =  \frac{\Psi_{\qde{v}}(\z)}{\Psi_{\qde{v}}(\z)} = 1
\end{align*}
proving that all of our induction assumptions hold and completing our proof of
the proposition.

\subsection{Proof of Theorem \ref{thm:spn_not_efficient}}
\label{app:proof:spn_not_efficient}

We heavily base our proof on \citet{Martens:2014tr}, who have proven a very
similar claim on a slightly different distribution on complete graphs, namely,
that SPNs cannot approximate the uniform distribution on the spanning trees of a
complete graph. Next, we go through the steps of their proof, citing the
relevant lemmas, and highlighting the places where our proof diverges.

We begin by citing the following decomposition lemma, paraphrased to match the
notations and definition of sec.~\ref{sec:pre}:
\begin{lemma}[paraphrase of theorem 39 of \citet{Martens:2014tr}]
\label{lemma:martens:spn_decomp}
    Suppose $\{\Psi_j(\x)\}_{j=1}^\infty$ are the respective outputs of a
    sequence of D\&C SPNs of size at most $s$ over $N$ binary variables, which
    converges point-wise (considered as functions of $\x$) to some function
    $\gamma$ of $\x$. Then we have that $\gamma$ can be written as:
    \begin{align}\label{eq:martens:spn_decomp}
        \gamma = \sum_{i=1}^k g_i h_i    
    \end{align}
    where $k \leq s^2$ and for all $i \in [k]$ it holds that $g_i$ and $h_i$ are
    real-valued non-negative functions of $\y_i$ and $\z_i$, respectively, where
    $\y_i$ and $\z_i$ are sub-sets / tuples of the variables in $\x$ satisfying
    that $\frac{N}{3} \leq |\y_i|, |\z_i| \leq \frac{2N}{3}$,
    $\y_i \cap \z_i = \emptyset$, and $\y_i \cup \z_i = \x$.
\end{lemma}
According to lemma~\ref{lemma:martens:spn_decomp}, it is sufficient to show that
if a function in the form of eq.~\ref{eq:martens:spn_decomp} is equal to a
strictly positive distribution of triangle-free graphs of $M$ vertices, denoted
by $d(\EEE)$, where $N = \binom{M}{2}$ is the number of variables representing 
the edges of the graph, then $k = 2^{\Omega(N)}$, because the $k$ is a lower
bound on the size of any SPN approximating~$d(\EEE)$.

Because the functions that comprise $\gamma$ are non-negative, then
$\gamma = 0$ if and only if for all $i$ it holds that $g_i h_i = 0$. Thus, if
$\gamma(\EEE) = d(\EEE) > 0$, i.e.~$\EEE$ represents a triangle-free graph, 
then either $g_i = 0$ or $h_i = 0$ on $\EEE$. We will prove that
$k = 2^{\Omega(N)}$ by showing that each term $g_i h_i$ can be non-zero on at
most a small fraction of the triangle-free graphs, and more specifically, that
it can be non-zero only on a small fraction of spanning trees, which are only a
sub-set of all triangle-free graphs.

Let $g$ and $h$ be functions as above, such that
$\frac{N}{3} \leq |\y|, |\z| \leq \frac{2N}{3}$, $\y \cap \z = \emptyset$, and
$\y_i \cup \z_i = \EEE$, and that $d(\EEE) = 0$ implies $g(\y) = 0$ or
$h(\z) = 0$. Examining the possible triangles of $\EEE$, we single out all the
triangles such that some of the edges are part of $\y$ and some of $\z$. Notice
that for such triangles the function $g \cdot h$ must employ a conservative
strategy, as each function on its own only see a part of the possible edges of
the triangle and hence cannot decide whether all edges are in the graph or not.
\citet{Martens:2014tr} call such triplet of edges \emph{constraint triangles},
and prove the following claims:
\begin{claim}[Paraphrase of proposition 42 of \citet{Martens:2014tr}]
\label{claim:martens:conservative_strategy}
    Let $E_{i_1 i_2}$, $E_{i_2 i_3}$, and $E_{i_1 i_3}$ be three different edges
    that form a constraint triangle with respect to $g$ and $h$ as above, for
    which if all edges are part of the graph then $g \cdot h = 0$. Additionally,
    suppose that both $E_{i_1 i_2}$ and $E_{i_2 i_3}$ are in the same set of
    variables with respect to the partition $\y \cup \z$. Then the following
    properties hold:
    \begin{itemize}
        \item $g(\y) \cdot h(\z) = 0$ for all values of $\EEE$ such that
              $E_{i_1 i_2} = 1$ and $E_{i_2 i_3} = 1$, i.e.~are part of the
              graph $\EEE$ represents.
        \item $g(\y) \cdot h(\z) = 0$ for all values of $\EEE$ such that
              $E_{i_1 i_3} = 1$, i.e.~is part of the graph $\EEE$ represents.
    \end{itemize}
\end{claim}
\begin{claim}[Paraphrase of lemma 43 of \citet{Martens:2014tr}]
\label{claim:martens:num_triangles}
    Given any partition of the edges of $\EEE$ into disjoint sets $\y\cup \z$,
    such that $\frac{N}{3} \leq |\y|, |\z| \leq \frac{2N}{3}$, then the total
    number of constraint triangles is lower bounded by $\frac{M^3}{60}$.
\end{claim}
Claim~\ref{claim:martens:conservative_strategy} means that if
$g(\y) \cdot h(\z) > 0$ then either $E_{i_1 i_2}$ and $E_{i_2 i_3}$ are not part
of the graph, or $E_{i_1 i_3}$ is not part of it, and thus each constraint
limits what graphs it can be non-zero on.
Claim~\ref{claim:martens:num_triangles} finds a lower bound on the number of
such constraints, which brings us to the following claim by
\citet{Martens:2014tr}, which finds an upper bound on percentage of spanning
trees that obey any given $C$ set of distinct constraints:
\begin{claim}[Paraphrase of lemma 44 of \citet{Martens:2014tr}]
\label{claim:martens:upper_bound}
    Suppose we are given $C$ distinct constraints which are each one of the two
    forms discussed above. Then, of all the spanning trees of the complete graph
    on $M$ vertices, a proportion of at most:
    \begin{align*}
        \left(1 - \frac{C}{M^3}\right)^{\nicefrac{C}{6M^2}}
    \end{align*}
    of them obey all of the constraints.
\end{claim}
Given that we have $C > \frac{M^3}{60}$, then it holds that
$g(\y) \cdot h(\z) > 0$ on at most $\frac{1}{2^{\nicefrac{M}{15120}}}$ of all
the possible spanning trees.

To conclude, $\gamma(\EEE)$ can be non-zero on at most
$\frac{k}{2^{\nicefrac{M}{15120}}}$ fraction of all spanning trees, and since
$d(\EEE)$ should be positive for any $\EEE$ that represents a triangle free
graph, such as any spanning tree, then if $\gamma(\cdot) = d(\cdot)$ it must be
that $\frac{k}{2^{\nicefrac{M}{15120}}} \geq 1$, which means
$k \geq 2^{\nicefrac{M}{15120}}$, or in other words, $s = O(2^{\Omega(M)})$.

\subsection{Proof of Theorem \ref{thm:spqn_are_efficient}}
\label{app:proof:spqn_are_efficient}

We start by examining all triangles for which the edge $E_{i_2 i_3}$ is the
largest edge (according to lexical order). For every $1 < i_2 < i_3 \leq M$,
we define the following variables:
\begin{align*}
    \varphi_{i_2 i_3}^{(1)} &\equiv\smashoperator[l]{\prod_{i_1 = 1}^{i_2 - 1}}
        \frac{\splitfrac{\indc{E_{i_1 i_2} {=} 1} \indc{E_{i_1 i_3} {=} 0}
            + \indc{E_{i_1 i_2} {=} 0} \indc{E_{i_1 i_3} {=} 1}}{
            + \indc{E_{i_1 i_2} {=} 0} \indc{E_{i_1 i_3} {=} 0}}}{3} \\
    \varphi_{i_2 i_3}^{(2)} &\equiv \smashoperator[lr]{\sum_{i_1 = 1}^{i_2 - 1}}
        \frac{1}{i_2 -1} \indc{E_{i_1 i_2}{=}1} \indc{E_{i_1 i_3}{=}1} \\
    &\phantom{\equiv\smashoperator[lr]{\sum_{i_1 = 1}^{i_2 - 1}}} {\cdot}
        \smashoperator[lr]{\prod_{i' \neq i_1}^{i_2 - 1}}
            \cramped{\frac{\indc{E_{i' i_2} {\shorteq} 0} {+} \indc{E_{i' i_2} {\shorteq} 1}}{2} }{\cdot }
            \frac{\indc{E_{i' i_3} {\shorteq} 0} {+} \indc{E_{i' i_3} {\shorteq} 1}}{2} \\
    \Phi_{i_2 i_3} &\equiv \frac{
            \frac{1}{2} \varphi_{i_2 i_3}^{(1)} \frac{\indc{E_{i_2 i_3} = 0} + \indc{E_{i_2 i_3} = 1}}{2} +
            \frac{1}{2} \varphi_{i_2 i_3}^{(2)} \indc{E_{i_2 i_3} = 0}
        }{
            \frac{1}{2}\varphi_{i_2 i_3}^{(1)} +
            \frac{1}{2}\varphi_{i_2 i_3}^{(1)}
        }
\end{align*}
Where $\varphi_{i_2 i_3}^{(1)}$ is a normalized probability
over the edges $E_{1,i_2},\ldots,E_{i_2-1,i_2}$ and
$E_{1,i_3},\ldots,E_{i_2-1,i_3}$, such that $\phi_{i_2 i_3}^{(1)}$ is non-zero
if and only if the edge $E_{i_2 i_3}$ cannot complete a triangle, i.e.~whether
$E_{i_2 i_3} = 0$ or $E_{i_2 i_3} = 1$ the graph can be triangle-free as long
as the other triplets of edges not containing $E_{i_2 i_3}$ do not result in a
triangle. Similarly, $\varphi_{i_2 i_3}^{(2)}$ is a normalized probability over
the same edges, but $\varphi_{i_2 i_3}^{(2)}$ is non-zero if and only if the
inclusion of the edge $E_{i_2 i_3}$ will necessarily complete one of the
triangles, i.e.~for the graph to be triangle-free then it must hold that
$E_{i_2 i_3} = 0$. Also notice that both $\varphi_{i_2 i_3}^{(1)}$ and
$\varphi_{i_2 i_3}^{(2)}$ can be defined by a D\&C SPN. Given the above,
either $\varphi_{i_2 i_3}^{(1)}>0$ or $\varphi_{i_2 i_3}^{(2)}>0$, hence the
denominator of $\Phi_{i_2 i_3}$ is always non-zero. Additionally, if
$\varphi_{i_2 i_3}^{(2)}$ is non-zero then $E_{i_2 i_3} = 0$ or else the graph
has a triangle, and otherwise either $E_{i_2 i_3} = 0$ or $E_{i_2 i_3} = 1$,
hence the numerator of $\Phi_{i_2 i_3}$ is greater than zero if and only if none
of the triangles considered are part of the graph. It is also trivial to verify
that the numerator is also a D\&C SPN, and hence conditionally D\&C. Finally, it
is clear from the construction that $\Phi_{i_{2},i_{3}}$ is strongly
conditionally sound, thus it is equivalent to a conditional distribution of
$E_{i_2 i_3}$ conditioned on the other edges of the triangles whose
$E_{i_2 i_3}$ is their largest edge, where
$\qeff{\Phi_{i_2 i_3}} = \{ E_{i_2 i_3} \}$ and
$\qcond{\Phi_{i_2 i_3}} = \{ E_{i_1 i_2}, E_{i_1 i_3} | 1 \leq i_1 < i_2 \}$.

With the above conditional distributions defined for all edges $E_{i_2 i_3}$
such that $1 < i_2 < i_3 \leq M$, we can now construct a strictly positive
distribution over triangle-free graphs. First, let us define
$\Phi_{1 i} \equiv \frac{\indc{E_{1 i} = 0} + \indc{E_{1 i} = 1}}{2}$ for all
$1 < i \leq M$, for which $\qeff{\Phi_{1 i}} = \{ E_{1 i}\}$ and
$\qcond{\Phi_{1 i}} = \emptyset$. Then, we define the probability as
$\Phi \equiv \prod_{1 \leq i < j \leq M} \Phi_{i j}$, and due to the definition
of $\Phi_{i_2 i_3}$ it is once more trivial to verify that $\Psi$ is
conditionally decomposable, and specifically that the induced dependency graph
is indeed cycle-free~--~this is due to the choice of lexical order which
guarantees that $E_{i_2 i_3}$ can only depend on edges which are smaller than
it, forbidding the formation of any cycle. In conclusion, $\Phi$ is a tractable
SPQN, which is non-zero if and only if the edges in $\E$ represent a
triangle-free graph~--~as required. Additionally, since the size of
$\varphi_{i_2 i_3}^{(1)}$ is at most $O(M)$, and the size of
$\varphi_{i_2 i_3}^{(2)}$ is at most $O(M^2)$, then the size of $\Phi$ is at
most $O(M^4)$, which proves the main result.

With regards to realizing the same SPQN with valid CMOs, notice that the
quotient nodes already follow the structure of valid CMOs, and that the
numerator and denominator are simply D\&C SPNs which SPQNs composed of valid
CMOs can arbitrarily approximate without changing the size of the model. Thus
this distribution can be approximated arbitrarily well with an SPQN composed of
valid CMO nodes of size at most $O(M^4)$.

\section{Experiments}\label{app:exp}

As a preliminary demonstration of the practical advantages of SPQNs over standard SPNs,
we have conducted a basic experiment on a synthetic dataset suited to the strengths of
SPQNs.

In essence, the difference between the two models is that SPNs have a limited
ability to represent intricate correlations between large set of variables~--~once a
sample is drawn for some variable, the result has no further effect on the rest of
the sampling process, i.e.~the drawing of the children of the remaining sum nodes.
In contrast, SPQN can due to the conditional distributions of its nodes.
We have came up with a simple synthetic dataset to demonstrate this difference,
comprising of $N \times N$ binary images generated as follows: first sample a random
location in the image, and then begin drawing a continuous non-overlapping path, where
at each step the path can be extended either forward, left, or right, with respect to
the direction of movement, with equal probability, and given that the next position is
free and not directly adjacent to a previous section of the line (excluding the current
position). See fig.~\ref{fig:dataset} for a selection of samples from this distribution
for $N = 8$. Since a pixel can be "on" only if there is a free path connecting it to
either ends of the drawn path, then it is dependent on all previously sampled pixels,
following our initial motivation.

\begin{figure}
\centering
\includegraphics[width=\linewidth]{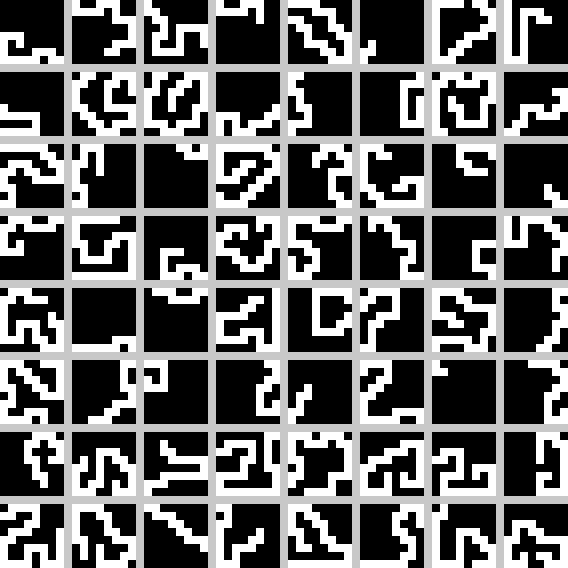}
\caption{Samples from the synthetic dataset we have designed to showcase the advantages
of SPQNs over SPNs.}
\label{fig:dataset}
\end{figure}

In our experiments we have randomly sampled from the above generative process for the
case of $N = 8$: 50000 examples for the training set, 1000 examples for validation and
10000 examples for the test set. We have trained SPNs with the structure learning algorithm
proposed by \citet{Gens:2013ufa}, where the hyper-parameters where chosen using grid-search
following the same space as in the original article. The best model had in total $189,128$
nodes.

For SPQNs we have first flatten the $8 \times 8$
binary image to a 1D array of size $64$, and then chosen a simple architecture mimicking a
1D convolutional network. Namely, the graph is composed of a sequence of ``convolutional''
layers, where each layer $d$ is defined by a stride $S_d$, receptive field $R_d$, and number of channels $C_d$,
and is composed of many CMO nodes spatially arranged and stacked according to $C_d$. For each
layer $d$, spatial position $t$, and channel $c$, there is a CMO node whose effective scope is
connected to nodes of layer $d-1$ at the spatial locations $t \cdot S - (S-1), \ldots, t \cdot S$
via intermediate sum nodes that are each connected the channels of a given spatial location, and 
similarly for the conditional scope at the spatial locations $t \cdot S - R + 1, t \cdot S - S$.
Essentially, the output $O_{d,t,c}$ of layer $d$ at location $t$ and channel $c$ is equivalent
to the following:
\begin{align*}
	O_{d,t,c} = \frac{
		\splitfrac{\sum\limits_{\mathclap{i=1}}^{C_d} W^{\text{Out}}_{c,i}
			\left(\prod\limits_{j=1}^S \sum\limits_{k=1}^{C_{d{-}1}} W^{\text{In}}_{i,j,k} O_{d-1,t \cdot S - j + 1,k}\right)
			}{
			\left(\prod\limits_{j=S+1}^R \sum\limits_{k=1}^{C_{d{-}1}} W^{\text{In}}_{i,j,k} O_{d-1,t \cdot S - j + 1,k}\right)}}{
		\sum\limits_{\mathclap{i=1}}^{C_d} W^{\text{Out}}_{c,i} \prod\limits_{j=1}^S \sum\limits_{k=1}^{C_{d{-}1}} W^{\text{In}}_{i,j,k} O_{d-1,t \cdot S - j + 1,k}}.
\end{align*}
The above architecture was trained with the Adam~\citep{Kingma:2014us} variant of SGD, using
$\beta_1 = \beta_2 = 0.9$, a learning rate of $5e-2$, mini-batches of 100 samples each, and
for 20 epochs. The other hyper-parameters where chosen using cross-validation, where the best
performing model was composed of 4 layers, with receptive fields equal to
$R_1 = R_2 = 32, R_3 = 16, R_4 = 1$, strides equal to $S_1 = S_2 = 2, S_3 = 16, S_4 = 1$,
and number of channels equal to $R_1 = R_2 = R_3 = 64, R_4 = 1$. In terms of sum, product,
and quotient nodes involved in the computation, it amounts to just $108,817$ nodes, on par
with the SPN model found via structure learning.

In the final results, the best SPN model attained a log-likelihood score of $-22.68$ on the
training set, $-24.35$ on the validation set, and $-24.71$ on the test set. In contrast,
the best SPQN model attained $-15.94$ on the training set, $-16.36$ on the validation set,
and $-16.51$ on the test set, which amounts to a $35\%$ improvement over SPNs. Given both models
are of similar size, and despite the fact no structure learning was used for SPQN model, then the
SPQN model clearly outperform by a large margin the standard SPN model. Nevertheless, it is
important to stress that further empirical evaluations are required to completely validate the
advantages of SPQNs over SPNs, even more so on real-world tasks.

\end{document}